\documentclass[twoside]{article}
\usepackage{amsmath,bbm}
\usepackage{amsthm}
\usepackage{amsfonts}
\usepackage{amssymb,color}
\usepackage{algorithm,algpseudocode,comment}      
\usepackage{url,graphicx}
\usepackage{caption}
\usepackage{subcaption}
\usepackage[accepted]{aistats2021}

\usepackage[round]{natbib}

\newtheorem{assumption}{\textsc{Assumption}}

\newcommand{\id}{\mathbbm{1}}
\newcommand{\bgi}{\bar{G}}
\newcommand{\bgin}{\bar{G}^{-1}}

\newcommand{\ust}{^{\star}}

\newcommand{\bR}{\mathbb{R}}
\newcommand{\bP}{\mathbb{P}}
\newcommand{\bN}{\mathbb{N}}

\newcommand{\cN}{\mathcal{N}}
\newcommand{\cA}{\mathcal{A}}
\newcommand{\cG}{\mathcal{G}}
\newcommand{\cB}{\mathcal{B}}
\newcommand{\cV}{\mathcal{V}}
\newcommand{\cT}{\mathcal{T}}

\newcommand{\cS}{\mathcal{S}}

\newcommand{\cU}{\mathcal{U}}
\newcommand{\cF}{\mathcal{F}}

\newcommand{\cC}{\mathcal{C}}
\newcommand{\cE}{\mathcal{E}}
\newcommand{\bE}{\mathbb{E}}
\newtheorem{theorem}{Theorem} 
\newtheorem{lemma}{Lemma} 
 
\newtheorem{definition}{Definition} 

\newtheorem{corollary}{Corollary} 

\begin{document}

\twocolumn[

\aistatstitle{Contextual Bandits with Side-Observations}

\aistatsauthor{ Rahul Singh \And Fang Liu \And  Xin Liu \And Ness Shroff }

\aistatsaddress{Department of ECE,\\
Indian Institute of Science\\
rahulsingh@iisc.ac.in \And  Department of ECE,\\Ohio State University\\
liu.3977@buckeyemail.osu.edu \And Department of CS,\\
University of California
Davis,\\ 
liu@cs.ucdavis.edu\And  Department of ECE,\\Ohio State University\\
shroff@ece.osu.edu   } ]

\begin{abstract}
We investigate contextual bandits in the presence of side-observations across arms in order to design recommendation algorithms for users connected via social networks. Users in social networks respond to their friends' activity, and hence provide information about each other’s preferences. In our model, when a learning algorithm recommends an article to a user, not only does it observe his\slash her response (e.g. an ad click), but also the side-observations, i.e., the response of his neighbors if they were presented with the same article. We model these observation dependencies by a graph $\cG$ in which nodes correspond to users, and edges correspond to social links.~We derive a problem\slash instance-dependent lower-bound on the regret of any consistent algorithm.~We propose an optimization based data-driven learning algorithm that utilizes the structure of $\cG$ in order to make recommendations to users and show that it is asymptotically optimal, in the sense that its regret matches the lower-bound as the number of rounds $T\to\infty$. We show that this asymptotically optimal regret is upper-bounded as $O\left(|\chi(\cG)|\log T\right)$, where $|\chi(\cG)|$ is the domination number of $\cG$. In contrast, a naive application of the existing learning algorithms results in $O\left(N\log T\right)$ regret, where $N$ is the number of users.
\end{abstract}
\section{Introduction}\label{sec:intro}
The contextual multi-armed bandit model is popularly used in order to place ads and make personalized recommendations of news articles to users of web services~\cite{linUCB,langford2008epoch}. 
In this model, both users and contents are represented by sets of features. For example, user features are obtained on the basis of their historical behavior and demographic information; while content feature depends upon its category and descriptive information. A learning algorithm for contextual bandits sequentially recommends articles to users based on contextual information of the articles and preferences of users, while continually adapting its strategy to present articles on the basis of feedback, e.g. ad clicks, downloads, etc., received from users. Its goal is to maximize the cumulative reward, which is equal to the total number of user clicks in the long run.\par
We consider the problem of making recommendations to users of a social network such as Facebook, Goodreads, LinkedIn. If users' preferences were known, we could employ an optimal stationary strategy that maps the context of each user to its optimal action (i.e., present her with an article that has the highest click-probability). Since users' preferences are typically unknown, one could employ an efficient contextual-bandit learning algorithm as in~\cite{lattimore2016end} on each user separately. This strategy achieves a regret of $O\left(N\log T\right)$, where $N$ is the number of users. However, since the number of users can be very large (e.g. Facebook has 2.5 billion users~\cite{wiki:Facebook}), this strategy is impractical. \par 
Consider a social network modeled by an undirected graph $\cG$ in which the nodes corespond to users, and undirected edges correspond to ``social links'', i.e., two users are connected by an edge if they are ``friends.'' Since individual users are connected to a subset of the remaining users, each time the algorithm makes a recommendation to a user, it also obtains feedback from her ``neighbors'' regarding their potential interest in a similar offer. For example, when a user $i$ is presented with a promotion $x$, his neighbors could be explicitly queried as follows: ``Would you be interested in promotion $x$ that was offered to your friend $i$?''. The response of user $i$'s friends to this query then constitutes \emph{``side observations''}. We design learning algorithms that incorporate these side-observations into the decision-making process for making recommendations. We show that the regret of the proposed algorithms scales at most as $O\left(|\chi(\cG) |\log T\right)$, where $|\chi(\cG) |$ is the domination number (see Definition~\ref{def:1}) of graph $\cG$. Since $|\chi(\cG)|\ll N$ for most graphs (see~\cite{wolf_dn} for more details), our algorithms drastically reduce the dependence of the regret on the number of users. \par 
In our setup, choosing an arm in the multi-armed bandit problem corresponds to making a recommendation to a \emph{single} user in network. We work with a \emph{linear} contextual bandit models, i.e., the one-step expected reward (e.g. user's ad-click probability) of an arm is a linear function of the context of arm. This (unknown) linear function depends upon the user on which this arm is played. For this contextual multi-armed bandit model with side observations, we derive a lower bound on the regret of any consistent algorithm, and show that our algorithms are asymptotically (as $T\to\infty$) optimal, since their asymptotic regret matches this lower bound.
\subsection{Related Work}
We begin by describing existing works on contextual bandits, and then discuss works that derive learning algorithms for models in which side-observations are present.

\emph{Contextual Bandits with Linear Pay-off Functions}: Contextual bandits with linear pay-off functions have been extensively studied. The efficiency of a learning algorithm is measured by its regret~\cite{bubeck,lattimore2020bandit}, which is the expected value of the difference between the cumulative reward collected by an algorithm that knows the true parameters of the problem instance and hence makes the optimal choice in each round, and the reward collected by the learning algorithm. Upper Confidence Bound (UCB)-based algorithms that use optimism in the face of uncertainty have been developed in works such as~\cite{auer2002using,linUCB,chu2011contextual}.  ~\cite{chu2011contextual} analyzes LinUCB  and shows that its minimax\slash worst-case regret scales as $\tilde{O}(\sqrt{dT})$, where $d$ is the dimension of feature space\footnote{$\tilde{O}(\cdot)$ hides factors that are logarithmic in number of rounds $T$.}.  ~\cite{chapelle2011empirical} and \cite{agrawal2013thompson} utilize Thompson sampling and prove that its regret scales as $\tilde{O}\left(\frac{d^2}{\epsilon}\sqrt{T^{1+\epsilon}}\right)$, where $0<\epsilon<1$. However, we focus on developing algorithms that have provably optimal problem-dependent regret guarantees~\cite{bubeck,lattimore2020bandit}. As has been shown in~\cite{lattimore2016end}, performance of learning algorithms based on UCB, or Thompson sampling can be arbitrarily far from optimal in this setting. 

Finite-time problem-dependent guarantees for linear bandits have been derived in~\cite{linear,abbasi2011improved}, however these are far from optimal.~\cite{lattimore2016end} 
studies problem-dependent regret in the asymptotic regime (when $T\to\infty$), and derives algorithm that is asymptotically optimal. We focus exclusively on the problem-dependent setting, and build upon techniques of~\cite{lattimore2016end}. 

\emph{Learning with Side Observations}: \cite{mannor} introduced the side-observation model in the adversarial multi-armed (non-contextual) bandit setting in which upon choosing an action, the decision maker not only receives reward from the chosen arm, but also gets to observe the rewards of its ``neighboring'' arms. The observation dependencies are encoded as an undirected graph $\cG$ in which two nodes $i,j$ are connected by an edge only if pulling an arm also reveals reward of the other arm.~\cite{caron2012leveraging},~\cite{buccapatnam2014stochastic} and \cite{DBLP:journals/jmlr/BuccapatnamLES17} extend results of~\cite{mannor} to the setup of stochastic multi-armed bandits in which rewards from an arm are i.i.d. across time, and reward distribution depends upon arm.~\cite{caron2012leveraging} derives algorithms whose regret scales as $O(|\gamma(\cG)|\log T)$, where $|\gamma(\cG)|$ is the clique cover number of the graph\footnote{A clique cover of a given undirected graph is a partition of the vertices of the graph into cliques, i.e., subsets of vertices within which every two vertices are adjacent. Clique cover number is the smallest number of cliques using which nodes of $\cG$ can be covered.} that describes observation dependencies.~\cite{buccapatnam2014stochastic} improves the regret to $O(|\chi(\cG)|\log T)$, where $|\chi(\cG)|$ is domination number of $\cG$. The key insight gained from~\cite{caron2012leveraging,buccapatnam2014stochastic} is that in the presence of side-observations, not only does an efficient algorithm need to take into account the history of rewards obtained from an arm, but also the location of the arm in the graph $\cG$. Thus, for example, it might even be optimal to pull an arm with a low estimate of mean reward, because it is connected to relatively unexplored arms, and the ``exploration gains'' resulting from side-observations outweight the (relatively larger) instantaneous regret of this arm. Our work generalizes the side-observations learning model to the case of contextual bandits with linear pay-offs. Note that in order to do so, the approach of~\cite{caron2012leveraging} or~\cite{buccapatnam2014stochastic} cannot be used to extend their results to the case of contextual bandits. This requires us to develop several novel tools and techniques. In effect, the resulting algorithms are vastly different from~\cite{caron2012leveraging} and \cite{buccapatnam2014stochastic}. The resulting prefactor, interestingly, turns out to be the domination number of the connectivity graph, which is same as in ~\cite{buccapatnam2014stochastic}. Our analysis builds upon the tools from~\cite{lattimore2016end}.
\subsection{Our Contributions}
Our main contributions are as follows:
\begin{enumerate}
\item We model the contextual multi-armed bandit problem in the presence of side-observations. We derive an instance dependent lower-bound on the regret of any consistent policy \footnote{A policy whose regret is smaller than $o(t^{p}),~\forall p>0$ and all possible instances of problem.}. This bound is the optimal value of an optimization problem that is parametrized by the graph that describes the side-observations dependencies, and the (unknown) sub-optimality gaps of arms. 
\item We propose a UCB-type learning algorithm that explores the values of unknown coefficients of users arms using a barycentric spanner of the set of context vectors of all the arms for each user in the network. It maintains confidence balls for the rewards of arms and uses a stopping rule in order to decide when to stop the exploration phase. At the end of exploration phase, it plays those arms that are optimal given the current estimates of coefficients. 
We analyze its finite time regret, and show that it can be upper-bounded as $O(|\cV|\log T)$ with a prefactor that depends upon the sub-optimality gaps of rewards of arms. This is a simple algorithm with a good regret bound, but does not match the lower bound. 

\item To close, the gap mentioned in 2 above, we develop a learning algorithm that is composed of three phases. During the \emph{warm-up phase}, it samples each user's coefficient vectors for fixed $O(\log^{1\slash 2}T)$ number of rounds. Thereafter, in the \emph{success phase} it uses these samples to estimate the unknown sub-optimality gaps of arms, and inputs these estimate into the optimization problem mentioned above. The solution of this problem then yields the number of times each arm is to be played. The algorithm uses a detector in order to constantly track the quality of estimates obtained at the end of warm-up phase. In the event it detects that these estimates are ``bad,'' it switches to the UCB-type algorithm described in 2. above. We show that this ``data-dependent" algorithm's regret asymptotically matches the lower-bound. 
\end{enumerate}

\section{Problem Setting}\label{sec:introduction}
The social network of interest is modeled by a graph $\cG=\left(\cV,\cE\right)$, in which the nodes $\cV$ represent users, while undirected edges $\cE$ represent social connections. We let $N:= |\cV|$ denote the number of users. Associated with each node $i\in\cV$ is a ``coefficient vector" $\theta\ust_i\in\bR^{d}$. In each round $t= 1,2,\ldots,T$, the decision maker recommends articles to each $i\in\cV$. Let $U_i(t)\in\cU \subset \bR^{d}$ denote the context of article presented to $i$ during $t$. Presenting article to a user $i$ also reveals ``side-observations" on its neighboring nodes $\cN_i:= \{j:(i,j)\in\cE\}$. These are rewards that would have been obtained if the same article was presented to users in the set $\cN_i$. We let $r_i(t)$ denote the reward received from recommendation to user $i$ during round $t$. Also let $y_{(i,j)}(t)$ denote the side-observation obtained from user $j$ as a result of recomendation to $i$ during round $t$.

We let $\mathcal{F}_t$ be the sigma-algebra generated by $\left\{ \{r_i(s)\}_{i\in\cV},\{y_{(i,j)}(s):j\in\cN_i\}_{i\in \cV},\{U_i(s)\}_{i\in\cV}\right\}_{s=1}^{t}$. Thus, it is the sigma-algebra generated by the operational history until round $t$. The reward earned from $i$ is given by
\begin{align}\label{def:reward}
r_i(t) = U^{T}_i(t)~\theta\ust_i +\eta_{i}(t),~i\in\cV, 
\end{align}
where $\eta_i(t)\sim \cN(0,1)$ is Gaussian and independent of $\mathcal{F}_t$.
Side-observations are given by, 
\begin{align}\label{def:side_obs}
y_{(i,j)}(t) = U^{T}_i(t)~\theta\ust_j +\eta_{(i,j)}(t), ~\forall (i,j)\in\cE,
\end{align} 
where $\eta_{(i,j)}(t)\sim \cN(0,1)$ are independent of $\cF_t$, and independent across social links. 

Note that we assume that the decision maker knows $\cG$.~The knowledge of graph is required in order for an algorithm to perform efficient exploration.~However, this is not a very restrictive assumption since the owner of social network, for example Facebook, has access to this information.

\textbf{Notation}: Denote $\theta\ust := \left(\theta\ust_1, \theta\ust_2,~\ldots,\theta\ust_N \right)\in \bR^{d\times N}$ the vector consisting of unknown coefficients. An ``arm'' $a$ that corresponds to playing context $u\in\cU$ on node (user) $i$ is denoted by the tuple $a=(i,u)$. For an arm $a$, we let $u_a$ denote its context vector, and $i_a$ its node. 
An optimal arm $b$ for node $i$ satisfies $b\in \arg\max_{a\in\cA_i}\left\{ u^{T}_{a}\theta\ust_i \right\}$. We assume that each node $i\in\cV$ has exactly one optimal arm, which is denoted by $a\ust_i$, and the corresponding optimal context vector is denoted $u\ust_i$. 

We let $\cA_i$ denote the set of arms that can be played on node $i$, and let $\cA:= \cup_{i\in\cV}\cA_i$ denote the set of all arms. $\cA^{(s)}_i$ denotes the set of all sub-optimal arms on node $i$, and $\cA^{(s)}:= \cup_{i\in\cV} \cA^{(s)}_{i}$ denotes the set of all sub-optimal arms. 
We also say that two arms $a_1=(i_1,u_1)$, $a_2=(i_2,u_2)$ are neighboring arms if $(i_1,i_2)\in\cE$. By notational abuse, we let $\cN_{a}$ denote the set of neighboring arms of arm $a$. When it is clear from the context, we will occasionally use $a$ to denote $u_a$. Define
\begin{align*}
\Delta_{a} := \max_{u\in\cU} u^{T}\theta\ust_{i_a} - u^{T}_a\theta\ust_{i_a},
\end{align*}
and also,
\begin{align}\label{def:del_mini}
\Delta_{\min,i} :&= \min_{a\in\cA^{(s)}_{i}} \left(a\ust_i\right)^{T}\theta\ust_{i} - a^{T}\theta\ust_{i},\\
\Delta_{\max,i} :&= \max_{a\in\cA^{(s)}_{i}} \left(a\ust_i\right)^{T}\theta\ust_{i} - a^{T}\theta\ust_{i},\\
\Delta_{\min} :&= \min_{a\in\cA}\Delta_{a},~\Delta_{\max} := \max_{a\in\cA}\Delta_{a}.
\end{align}

In what follows, we assume that a node is a neighbor of itself, i.e., $i\in\cN_i$, or equivalently $(i,i)\in\cE,~\forall i\in\cV$. Thus, we let $y_{(i,i)}(s) = r_i(s)$. This notation drastically simplifies the exposition.

All vectors are assumed to be column vectors. $0_{m\times n}$ denotes an $m\times n$ matrix comprises of only zeros. For a matrix $M$, $tr(M)$ or $M^{T}$ will denote its transpose, while $trace(M)$ denotes its trace, and $col_k(M)$ denotes its $k$-th column. For two vectors $x,y\in\bR^{d}$, $\left<x,y \right>$ denotes the dot product between $x$ and $y$. We will use $<x,y>$ and $x^{T}y$ interchangebly for dot product between $x,y$. We let $N_{a}(t)$ denote the number of times arm $a$ has been played until round $t$. For a vector $x\in\bR^{d}$ we let $\|x\|$ denote its Euclidean norm, and for a positive-definite matrix $H$, we let $\|x\|^{2}_{H}:= x^{T}Hx$. For two integers $m,n$ satisfying $n>m$, we let $[m,n]:= \left\{m,m+1,\ldots,n\right\}$. 

\textbf{Learning Algorithm and Regret}: A learning algorithm $\pi:  \cF_t \mapsto \otimes_{i\in\cV} \cA_i, t=1,2,\ldots,T$ maps the observational history until each round $t$, to a set of $|\cV|$ arms, one for each node. As discussed earlier, we let $U_i(t)$ denote the context of arm chosen for node $i$. The performance of $\pi$ is measured by its regret $R^{\pi}_{(\theta\ust,\cG,\cA)}(T)$,
\begin{align}
R^{\pi}_{(\theta\ust,\cG,\cA)}(T) := \bE \sum_{t=1}^{T} \left(\sum_{i\in\cV} (u\ust_i)^{T} \theta\ust - U^{T}_i(t)\theta\ust_i\right),
\end{align}
where the expectation above is taken with respect to the probability measure induced by the algorithm $\pi$, and randomess of rewards. Our objective is to design a learning algorithm that has a low regret.

\begin{definition}(Consistent Algorithm)
A learning algorithm $\pi$ is called consistent if for all $\theta\ust,\cA,\cG$ and $p>0$, it satisfies $R^{\pi}_{(\theta\ust,\cG,\cA)}(T) = o\left(T^{p}\right)$. 
\end{definition}
\subsection{Preliminaries}\label{sec:assum}
Now we introduce the concept of a barycentric spanner, and generalize it to the graphical setting, which will be crucial in our algorithm design. The following can be found in~\cite{DBLP:conf/stoc/AwerbuchK04}.
\begin{definition}[Barycentric Spanner of $\cU$]\label{def:bs1}
A set of context vectors $\cC\subseteq \cU$ is called barycentric spanner of~~$\cU$ if each $u\in\cU$ can be written as follows,
\begin{align*}
u = \sum_{w\in \cC} \alpha_{w} ~w,\quad \mbox{ where } \alpha_{w}\in [-1,1]. 
\end{align*}
\end{definition}
The following result is Proposition 2.2 and Proposition 2.4 of~\cite{DBLP:conf/stoc/AwerbuchK04}.
\begin{lemma}\label{lemma:bs}
There exists a barycentric spanner of $\cU$ that has cardinality less than or equal to $d$. Moreover it can be obtained in time polynomial in $d$. 
\end{lemma}

\begin{definition}[Barycentric Spanner of $\left(\cA,\cG\right)$]\label{def:bs2}
Let $\cC$ be a barycentric spanner of~~$\cU$. Then, the set of arms $\cS$,
\begin{align*}
\cS := \left\{ (i,u): i\in\cV,u\in\cC\right\},
\end{align*}
is a barycentric spanner of $(\cA,\cG)$. In what follows, we let $\cS$ be such a barycentric spanner of cardinality $Nd$.
\end{definition}

\begin{definition}(Dominating set of a graph)\label{def:1}
A dominating set of graph $\cG=\left(\cV,\cE\right)$ is a set of nodes such that every node from $\cV$ is either in this set, or is a neighbor of some node belonging to this set. Let $\chi(\cG)$ be a dominating set with minimum cardinality. $|\chi(\cG)|$ is called the domination number of $\cG$.
\end{definition}

\section{Lower Bounds}\label{sec:lb}
Define
\begin{align}\label{def:git}
G_i(t) :&= \sum_{s=1}^{t} \sum_{j\in\cN_i} U_j(s) U^{T}_j(s),\\
\bar{G}_i(t) :&= \bE\left(G_i(t)\right),~\forall i\in\cV.
\end{align}
We have the following lower bound on the regret of any consistent learning algorithm.~Auxiliary results required while proving it are deferred to the Appendix.
\begin{theorem}\label{th:main_res}
Under any consistent learning algorithm, we have
\begin{align}\label{eq:lower_bound}
\limsup_{T\to\infty}~ \log(T) \|u_{a}\|^{2}_{\bgin_{i_a}(T) } \le \frac{\Delta^{2}_{a}}{2},~~\forall a \in\cA^{(s)}.
\end{align}
Consider the following optimization problem,
\begin{align}
OPT : &\min_{\left\{\alpha(a):a\in\cA^{(s)}\right\}} \sum_{a\in\cA^{(s)} } \alpha(a) \Delta_{a} \label{lp_1}\\
&\mbox{ s.t. }  \|u_a\|^{2}_{H^{-1}_{i_a}(\alpha)} \le \frac{\Delta^{2}_{a}}{2},  ~~\forall a \in\cA^{(s)},\label{lp_2}\\
&\mbox{ where }H_i(\alpha) := \sum_{j\in\cN_i} \sum_{\left\{a: i_a = j\right\}}\alpha(a) a a^{T},\label{lp_3}
\end{align} 
where we have $\alpha = \{\alpha(a)\}_{a\in\cA}$, and $\alpha(a)\in [0,\infty),~\forall a$. Let $c(\theta\ust,\cG,\cA)$ denote its optimal value. We then have 
\begin{align}\label{ineq:lp_bound}
\limsup_{T\to\infty} \frac{R(T)}{\log T} \ge c(\theta\ust,\cG,\cA).
\end{align}
\end{theorem}
Note that solving $OPT$ requires us to know the values $\Delta_{a}$. 
\begin{proof}
We begin by proving~\eqref{eq:lower_bound}. Consider a sub-optimal arm $a\in\cA^{(s)}_i$. Recall that $a\ust_i$ is the optimal arm at node $i$. We have,
\begin{align}
\|u_a\|_{\bgin_{i}(T)} &\le \|u_a- a\ust_i\|_{\bgin_{i}(T)}  +  \| a\ust_i\|_{\bgin_{i}(T)} \notag\\
&\le \|u_a- a\ust_i\|_{\bgin_{i}(T)}  + \frac{ \| a\ust_i\|}{\sqrt{N_{a\ust_i}(T) }},\label{ineq:4}
\end{align}
where the first inequality follows from the triangle inequality, while the second follows since from~\eqref{def:git} we have that $G_i(T) \ge N_{a\ust_i}(T) ~a\ust_i \left(a\ust_i\right)^{T}$, which yields $\bar{G}^{-1}_i(T) \le \left(N_{a\ust_i}(T) \right)^{-1}  \left[a\ust_i \left(a\ust_i\right)^{T} \right]^{\dagger}$, where for a matrix $A$, we let $A^{\dagger}$ denote its pseudoinverse. After multiplying both sides of the inequality~\eqref{ineq:4} by $\log^{1\slash 2} T$, we obtain
\begin{align*}
&\limsup_{T\to\infty}  ~\log^{1\slash 2} T\|u_a\|_{\bgin_{i}(T)} \\
&\le \limsup_{T\to\infty} \log^{1\slash 2} T \|u_a- a\ust_i\|_{\bgin_{i}(T)} \\
&+ \limsup_{T\to\infty}  \frac{\log^{1\slash 2} T}{\sqrt{N_{a\ust_i}(T) }}  \| a\ust_i\|.
\end{align*}
Under a consistent policy, we have $\lim_{T\to\infty} \frac{N_{a\ust_i}(T)}{T} = 1$, so that the second term on the r.h.s. vanishes. It follows from Lemma~\ref{lemma:bound}~ that the first term on the r.h.s. is upper-bounded by $\Delta_{a}\slash \sqrt{2}$. Substituting these into the above inequality yields the proof of~\eqref{eq:lower_bound}.

We now show~\eqref{ineq:lp_bound}. Let $\pi$ be a consistent policy, and let $N_a(T)$ denote the number of times it pulls an arm $a$ until round $T$. Define $\alpha^{(T)}(a):=\frac{\bE N_{a}(T)}{\log T}$, and denote $\alpha^{(T)}:=\left\{\alpha^{(T)}(a): a\in\cA\right\}$. Its regret $R^{\pi}(T)$ satisfies, 
\begin{align}\label{ineq:16}
\frac{R^{\pi}(T) }{\log T}= \sum_{a\in\cA^{(s)}} \alpha^{(T)}(a)~\Delta_{a}. 
\end{align}

Note that $\bar{G}_i(T)= (\log T) H_i(\alpha^{(T)})$ or~$\bgin_{i}(T) = (\log T)^{-1} H^{-1}_{i}(\alpha^{(T)})$, where the function $H_i(\cdot)$ is as defined in~\eqref{lp_3}. Since $\pi$ is consistent, it then follows from~\eqref{eq:lower_bound} that $~\forall a\in\cA^{(s)}$ we have,
\begin{align}\label{ineq:17}
\limsup_{T\to\infty}  \|u_a\|^{2}_{H^{-1}_{i_a}(\alpha^{(T)})} = \limsup_{T\to\infty} \log T \|u_a\|^{2}_{\bgin_{i_a}(T)} \le \frac{\Delta^{2}_{a}}{2}.
\end{align}
Let $\alpha^{(\infty)} = \left\{ \alpha^{(\infty)}(a) : a\in\cA  \right\}$ be a limit point of $\alpha^{(T)}$. It follows from~\eqref{ineq:17} that the vector $\alpha^{(\infty)}$ is feasible for~\eqref{lp_1}-\eqref{lp_3}. Moreover, it follows from~\eqref{ineq:16} that the regret $R^{\pi}(T)$ satisfies   
\begin{align*}
\limsup_{T\to\infty}\frac{R^{\pi}(T) }{\log T}\ge \sum_{a\in\cA^{(s)}} \alpha^{(\infty)}(a)~\Delta_{a}.
\end{align*}
This completes the proof. 
\end{proof}

\section{Stopping Time based Algorithm}\label{section:finite}
\begin{algorithm}
\begin{algorithmic}
\State {\bfseries Input:} Arms $\cA$, Graph $\cG$, Confidence parameter $\delta$, Time horizon $T$
\State {\bfseries Initialize:} Set $t:=1$, and estimates $\hat{\theta}(t) = \left(1,1,\ldots,1\right)$ for all $i\in\cV$
\State \textcolor{blue}{// Exploratory Phase}
\While{ $\exists i\in \cV$ such that $\cB^{(o)}_{i,1}(t)\cap \cB^{(o)}_{i,m}(t)\neq\emptyset$ for some $m$  }
\State Play arms $a\in \cS$ in round-robin fashion
\State Update the estimates $\hat{\theta}_i(t)$ using~\eqref{def:empirical}
\State Update the confidence balls $\cB_{a}(t)$ using~\eqref{def:balls}
\EndWhile
\State Obtain estimates $\hat{\theta}_i(t)$ of the coefficients, and the optimal arms $\hat{a}\ust_i(\tau), i\in\cV$ 
\State \textcolor{blue}{// Exploitation Phase}
\For{$t=\tau+1,\tau+2,\ldots,T$}
\State Play $\left\{\hat{a}\ust_i(\tau)\right\}_{i\in\cV}$ on corresponding nodes
\EndFor
\end{algorithmic}
\caption{Stopping Time Based Algorithm}
\label{algo:stop}
\end{algorithm}
We now propose an algorithm for contextual bandits with side-observations. This algorithm is composed of two phases (i) Exploratory Phase, that is followed by (ii) Exploitation phase. The exploratory phase lasts until a stopping criteria is met. More details are as follows. 

\emph{Exploratory Phase}: Only the arms in the barycentric spanner $\cS$ are played in a round robin manner. Since $\cS$ is composed of $d$ arms at each node $i$, this phase is composed of sets of consecutive rounds of the form $\left[kd+1,(k+1)d\right], k=0,1,\ldots$ such that each arm in $\cS$ is played exactly once during each such set. We call each such set an episode. The algorithm maintains the empirical estimates $\{\hat{\theta}_i(t)\}_{i\in\cV}$ of the unknown coefficients $\theta\ust_i$, which are obtained as follows,
\begin{align}\label{def:empirical}
\hat{\theta}_i(t) :=G^{-1}_i(t) \left[\sum_{s=1}^{t}\sum_{j\in\cN_i} y_{(j,i)}(s)U_{j}(s)\right], i\in\cV,
\end{align}
where $G_i(t)$ is as in~\eqref{def:git}.
Additionally, it also maintains confidence ball $\cB_{a}(t)$ around the estimate of mean reward of each arm $a$ as follows,
\begin{align}\label{def:balls}
\cB_{a}(t) := \left\{ \mu\in\bR:  | \mu - \hat{\theta}_{i_a}(t)  | \le \alpha(t)    \right\}, a\in \cA,
\end{align}
where 
\begin{align}\label{def:alpha_1}
\alpha(t) := \sqrt{\frac{2\log\left(T \sum_{i\in\cV} |\cA_i|\slash \delta\right)}{t }~d}.
\end{align}

It orders the balls $\left\{\cB_{a}(t)\right\}_{a\in\cA_{i}}$ at each node $i$ in decreasing order of the corresponding values of the estimates of the mean rewards $\left\{a^{T} \hat{\theta}_{i}(t): a\in\cA_i\right\}$. Let $\cB^{(o)}_{i,m}(t)$ be the $m$-th such ball\footnote{Superscript denotes ordered balls.} at node $i$ during round $t$. Define $\tau_i$ to be the following stopping time, 
\begin{align}
&\tau_i:= \inf \big\{ t:  t= kd\mbox{ where } k\in\bN,\notag\\
&~ \cB^{(o)}_{i,1}(t)\cap \cB^{(o)}_{i,m}(t)=\emptyset,~~\forall m=2,3,\ldots,|\cA_i|  \big\},\label{def:stop_1}\\
\mbox{ and }\tau :&= \max_{i\in\cV} \tau_i.\label{def:stop_2}
\end{align} 
Exploratory phase ends at round $\tau$.

\emph{Exploitation Phase}: Let $\hat{a}\ust_i(t)$ be the optimal arm for node $i$ when the true value of coefficient of $i$ is equal to $\hat{\theta}_i(t)$, i.e., $\hat{a}\ust_i(t) \in \arg\max_{a_i\in\cA_i} \left\{ a^{T}_i\hat{\theta}_i(t) \right\}, ~~i\in\cV$. During rounds $t>\tau$, algorithm plays only the arms $\left\{\hat{a}\ust_i(\tau), i\in\cV\right\}$ at their corresponding nodes. Thus, it uses $\left\{\hat{\theta}_i(\tau)\right\}_{i\in\cV}$ as a proxy for $\theta\ust$, and plays the resulting greedy decisions.\par
The following result provides upper bound on regret of Algorithm~\ref{algo:stop}. We defer its proof to the appendix.
\begin{theorem}\label{th:1}
The regret $R(T)$ of Algorithm~\ref{algo:stop} is upper-bounded as
\begin{align*}
R(T) &\le \left( \sum_{i\in\cV} \Delta_{\max,i} \right) \frac{2\log\left(T \sum_{i\in\cV} |\cA_i|\slash \delta\right)d}{\left(\Delta_{\min}\slash 2\right)^{2}}\notag\\
&+\delta T\left(\sum_{i\in\cV} \Delta_{\max,i}\right).
\end{align*}
With $\delta=1\slash T$, we obtain the following upper-bound on regret,
\begin{align*}
R(T) &\le \left( \sum_{i\in\cV} \Delta_{\max,i} \right) \frac{2\log\left(T^{2} \sum_{i\in\cV} |\cA_i|\right)d}{\left(\Delta_{\min}\slash 2\right)^{2}}\\
&+\left(\sum_{i\in\cV} \Delta_{\max,i}\right).
\end{align*}
\end{theorem}
Although Algorithm~\ref{algo:stop} takes an ``explore and commit approach'' using a naive exploration algorithm, and consequently it seems suboptimal, it is required in the ``bad event'' when the estimates of sub-optimality gaps that have been obtained at the end of the ``warm-up phase'' turn out to be bad. In this event, unless an algorithm such as Algorithm~1 is deployed in this bad event, the expected regret would be quite large since the regret on this bad set (event) could be as large as the time horizon $T$. 
\section{Asymptotically Optimal Algorithm}\label{sec:lp_algo}
Note that the regret of Algorithm~\ref{algo:stop} scales as $O(\log T)$; if the parameters $\Delta_{\max},\Delta_{\min}$ and the dimension of contexts $d$ are kept constant, then the regret scales linearly with the number of nodes $N$. We now propose an algorithm that uses solution of~\eqref{lp_1}-\eqref{lp_3} in order to make decisons. Its regret matches the lower bound of Section~\ref{sec:lb}, and the prefactor can be upper-bounded by domination number $|\chi(\cG)|$ of graph $\cG$. For graphs $\cG$ that satisfy $|\chi(\cG)|\ll |V|$, this algorithm can be much more efficient than Algorithm~\ref{algo:stop}. 

\begin{algorithm}
\begin{algorithmic}
\State {\bfseries Input:} Arms $\cA$, Graph $\cG$, Confidence parameter $\delta$, Time horizon $T$
\State \textcolor{blue}{// Warm-up Phase}
\State Play each arm in spanning set $\cS$ for $\log^{1\slash2} T$ times 
\State \textcolor{blue}{// Success Phase}
\State $\epsilon_{T}(t)\leftarrow\max_{a\in\cA}  \|a\|_{G^{-1}_{i_a}(d\log^{1\slash2} T)} ~~g^{1\slash 2}(T)$
\State $\hat{\Delta} \leftarrow \hat{\Delta}(d\log^{1\slash2} T), \hat{\mu}\leftarrow \hat{\mu}(d\log^{1\slash2} T)$
\State Solve $OPT(\hat{\Delta} )$~\eqref{def:lph_o}-\eqref{def:lph_c2} to obtain $\beta\ust(\hat{\Delta})$
\While{$t\le T$ and $ |\hat{\mu}_a-\hat{\mu}_a(t-1)|\le 2\epsilon_{T} $ for all $a\in\cA$}
\State For each $i\in\cV$ play actions in a round robin fashion with $N_a(t) \le \beta\ust_{a}(\hat{\Delta})$
\EndWhile
\State \textcolor{blue}{// Recovery Phase}
\State Discard all data and play $\pi^{ball}$ until $t=T$
\end{algorithmic}
\caption{Asymptotically Optimal Algorithm}
\label{algo:opt}
\end{algorithm}

We begin by introducing some notations. Define,
\begin{align}
f(t) :&= 2\log(t) + cd\log\left(d\log t\right) +2,\label{def:ft}\\
g(t) :&= 2 \log\left(\log t\right) + 2\frac{\log\left(\log t\right)}{\log t} + cd \log(d\log t ),\label{def:gt}
\end{align} 
where $c>0$ is a constant. Let
\begin{align}\label{def:sog}
\hat{\Delta}_{a}(t):= \max_{b\in\cA_{i_a}}\left(b-a\right)^{T} \hat{\theta}_{i_a}(t),
\end{align}
denote estimate of sub-optimality gap of arm $a$ during round $t$, and $\hat{\Delta}(t) := \left\{\hat{\Delta}_{a}(t)\right\}_{a\in\cA}$. \\
The proposed algorithm is composed of three phases.\par
\emph{Warm-up Phase}: Algorithm plays arms in barycentric spanner $\cS$ for $d\log^{1\slash 2} T$ rounds. 

\emph{Success Phase}: Denote by $\hat{\Delta}=\{\hat{\Delta}_a\}_{a\in\cA},\hat{\mu} = \{\hat{\mu}_a\}_{a\in\cA}$ the estimates of sub-optimality gaps and mean values of rewards that are obtained by using the information gained during warm-up phase. 

Consider the following optimization problem obtained from $OPT$~\eqref{lp_1}-\eqref{lp_3} by replacing the gaps $\Delta_{a}$ by their estimates $\hat{\Delta}= \left\{\hat{\Delta}_{a}: a\in \cA\right\}$:
\begin{align}
OPT(\hat{\Delta}): \min_{\left\{\beta_a\right\}_{a\in\cA}} &\sum_{a\in\cA}  \beta_a \hat{\Delta}_{a} \label{def:lph_o}\\
\mbox{ s.t. } f(T) \|u_a\|^{2}_{H^{-1}_{i_a}(\beta)} &\le \frac{\hat{\Delta}^{2}_{a}}{2},  ~~\forall a \in\cA,\label{def:lph_c1}\\
\mbox{ where }H_i(\beta) :&= \sum_{j\in\cN_i} \sum_{\left\{a: i_a = j\right\}}\beta_a~ a~ a^{T},~i\in\cV.\label{def:lph_c2}
\end{align} 
Let $\beta\ust(\hat{\Delta})= \left\{ \beta\ust_{a}(\hat{\Delta}) : a\in\cA\right\}$ be a solution of $OPT(\hat{\Delta})$.

The algorithm uses estimates $\hat{\Delta}$ to solve~\eqref{def:lph_o}-\eqref{def:lph_c2}, and obtains $\beta\ust(\hat{\Delta})$. It then plays each arm $a$ in a round-robin fashion until it has been played for $\beta\ust_{a}(\hat{\Delta})$ rounds. 
Meanwhile, it also continually keeps track of the quality of estimates $\hat{\mu}_a$ of rewards obtained at the end of warm-up phase as follows. 
Define
\begin{align}\label{def:eps_t}
\epsilon_{T}(t):= \max_{a\in\cA}  \|a\|_{G^{-1}_{i_a}(t)} ~~g^{1\slash 2}(T),
\end{align}
where $g(\cdot)$ is as in~\eqref{def:gt}. If during any round $t$, it observes that $|\hat{\mu}_{a}(t)-\hat{\mu}_{a}|>2\epsilon_{T}(d\log^{1\slash 2} T)$ for some arm $a\in\cA$, then it declares that the estimates $\hat{\mu}$ are bad, and in this event algorithm enters recovery phase.

\emph{Recovery Phase}: Algorithm discards all operational history and collected data, and starts playing Algorithm~\ref{algo:stop}. 

We next show that this algorithm is asymptotitcally optimal, i.e., as $T\to\infty$, its regret matches the lower bound derived in Theorem~\ref{th:main_res}.~Auxiliary results required while proving it are deferred to the Appendix.
\begin{theorem}\label{th:regret_opt}
The regret $R(T)$ of Algorithm~\ref{algo:opt} satisfies
\begin{align*}
\limsup_{T\to\infty} \frac{R(T)}{\log T} \le c(\cA,\theta\ust,\cG),
\end{align*}
where $c(\cA,\theta\ust,\cG)$ is the optimal value of optimization problem~\eqref{lp_1}-\eqref{lp_3}. It then follows from lower bound derived in Theorem~\ref{th:main_res} that Algorithm~\ref{algo:opt} is asymptotically optimal as $T\to\infty$.
\end{theorem}
\begin{proof}Throughout this proof, we denote the regret of Algorithm~\ref{algo:opt} by $R(T)$. Let $\cT_{warm},\cT_{succ},\cT_{rec}$ denote the rounds spent in the warm-up, success, and recovery phases respectively. The normalized cumulative regret $R(T)\slash \log T$ can be decomposed as follows,
\begin{align}\label{eq:regret_dec}
&\frac{R(T)}{\log T} = \frac{1}{\log T}~\bE\left(  \sum_{t\in \cT_{warm}} \sum_{i\in\cV} \Delta_{U_i(t)} \right) \notag\\
&+ \frac{1}{\log T}~\bE\left(\sum_{t\in \cT_{succ}} \sum_{i\in\cV} \Delta_{U_i(t)} +\sum_{t\in \cT_{rec}} \sum_{i\in\cV} \Delta_{U_i(t)}  \right). 
\end{align}
Since the warm-up phase lasts for $O(\log^{1\slash 2} T)$ rounds, contribution of the first term is asymptotically $0$ as $T\to\infty$.\par 
Next, we analyze the regret during $\cT_{rec}$. It follows from Lemma~\ref{lemma:5} and Lemma~\ref{lemma:2} that $\bP\left(\cT_{rec} \neq \emptyset\right)\le 1\slash \log T$. Also, from Theorem~\ref{th:1} we have that if the algorithm does enter the recovery phase, then its regret is upper-bounded as $O(\log T)$. Combining these two bounds, we have that the expected value of last term in summation in~\eqref{eq:regret_dec} is upper-bounded by a constant that does not depend upon $T$. Thus, the contribution of this summation term, when divided by $\log T$, is asymptotically $0$.\par 
The discussion so far shows that the first and the last summation terms in the r.h.s. of~\eqref{eq:regret_dec} asymptotically vanish.~We finally analyze the regret in the success phase. We further decompose this term as follows,
\begin{align}\label{eq:3}
&\frac{\bE\left[ \sum_{t\in \cT_{succ}} \sum_{i\in\cV} \Delta_{U_i(t)} \right] }{\log T} \notag\\
&= \frac{\bE\left[ \mathbbm{1}(F^{c}) \sum_{t\in \cT_{succ}} \sum_{i\in\cV} \Delta_{U_i(t)} \right]}{\log T}\notag\\
&+ \frac{\bE\left[ \mathbbm{1}\left(F\cap \left(F^{\prime}\right)^{c}\right) \sum_{t\in \cT_{succ}} \sum_{i\in\cV} \Delta_{U_i(t)} \right]}{\log T}\notag\\
&+ \frac{\bE\left[ \mathbbm{1}\left(F^{\prime}\right) \sum_{t\in \cT_{succ}} \sum_{i\in\cV} \Delta_{U_i(t)} \right]}{\log T}.
\end{align}
It follows from Lemma~\ref{lemma:3} that the second term on the r.h.s. above asymptotically vanishes.~Since from Lemma~\ref{lemma:5}, we have that $\bP(F^{\prime}) \le 1\slash T$, and moreover, the regret on any sample-path can be trivially upper-bounded as $O(T)$, we conclude that $\bE\left[\id \left(F^{\prime}\right) \sum_{t\in \cT_{succ}} \sum_{i\in\cV} \Delta_{U_i(t)} \right]$ is upper-bounded by a constant that does not depend upon $T$. Thus, the last term in the r.h.s. above also vanishes asymptotically. Finally, as shown in Lemma~\ref{lemma:reg_fc}, the first term in the r.h.s. is asymptotically upper-bounded by $c(\theta\ust,\cG,\cA)$. Substituting these three bounds into the relation~\eqref{eq:3} completes the proof. 
\end{proof}
\begin{corollary}\label{coro:2}
Optimal value of problem~\eqref{lp_1}-\eqref{lp_3}, $c(\cA,\theta\ust,\cG)$, is less than or equal to $\frac{\Delta_{\max}}{\Delta_{\min}}|\chi(\cG)|$.
Thus, the regret $R(T)$ of Algorithm~\ref{algo:opt} satisfies
\begin{align*}
\limsup_{T\to\infty} \frac{R(T)}{\log T} \le \frac{\Delta_{\max}}{\Delta_{\min}}|\chi(\cG)|.
\end{align*}
\end{corollary}

\section{Experiments}
\begin{figure}
\centering
\begin{subfigure}[b]{0.42\textwidth}
    \includegraphics[width=\textwidth]{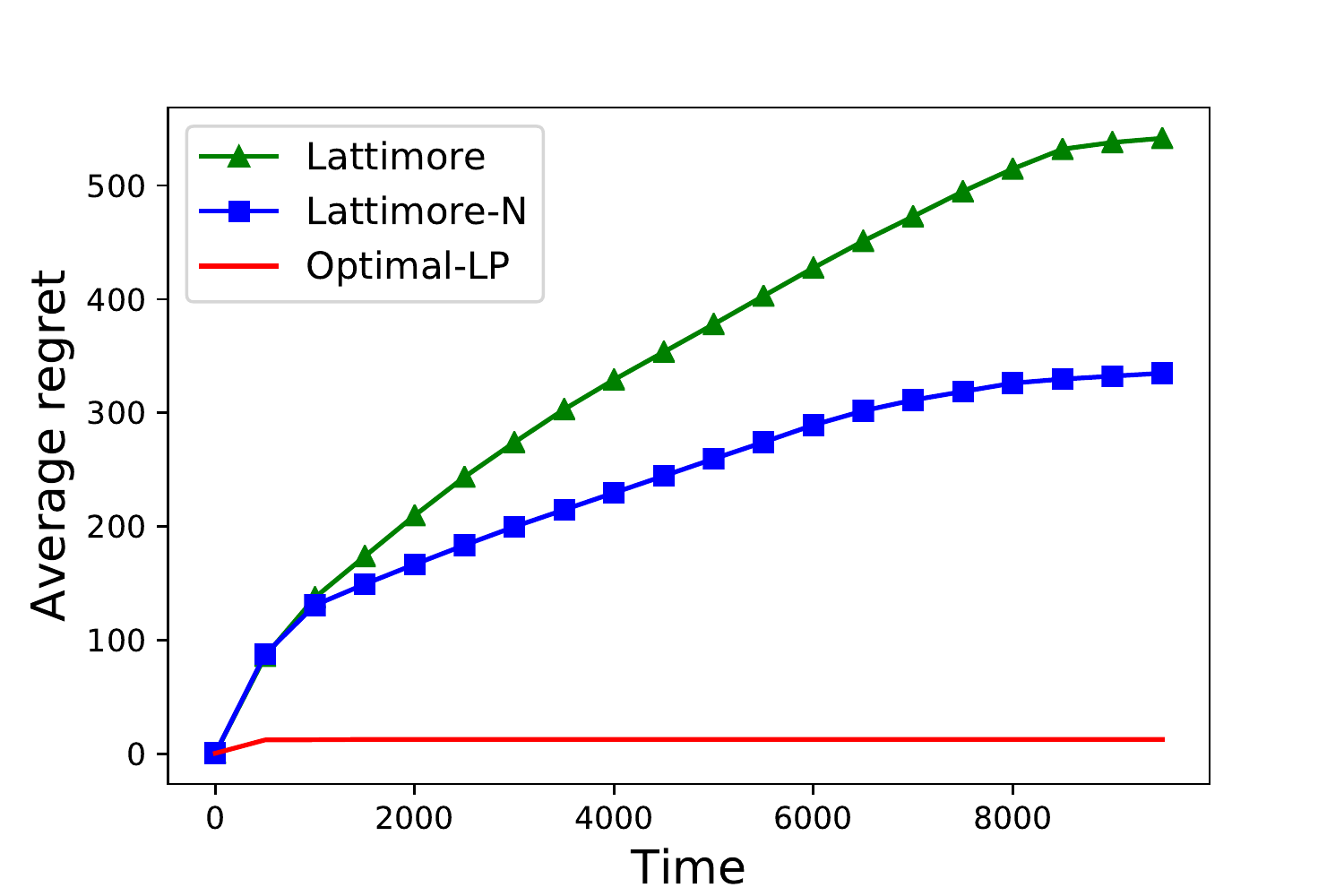}
     \caption{$d = 2, N =2, K=5, p=1$}
     \label{fig:synthetic_1}
\end{subfigure}
\begin{subfigure}[b]{0.42\textwidth}
    \includegraphics[width=\textwidth]{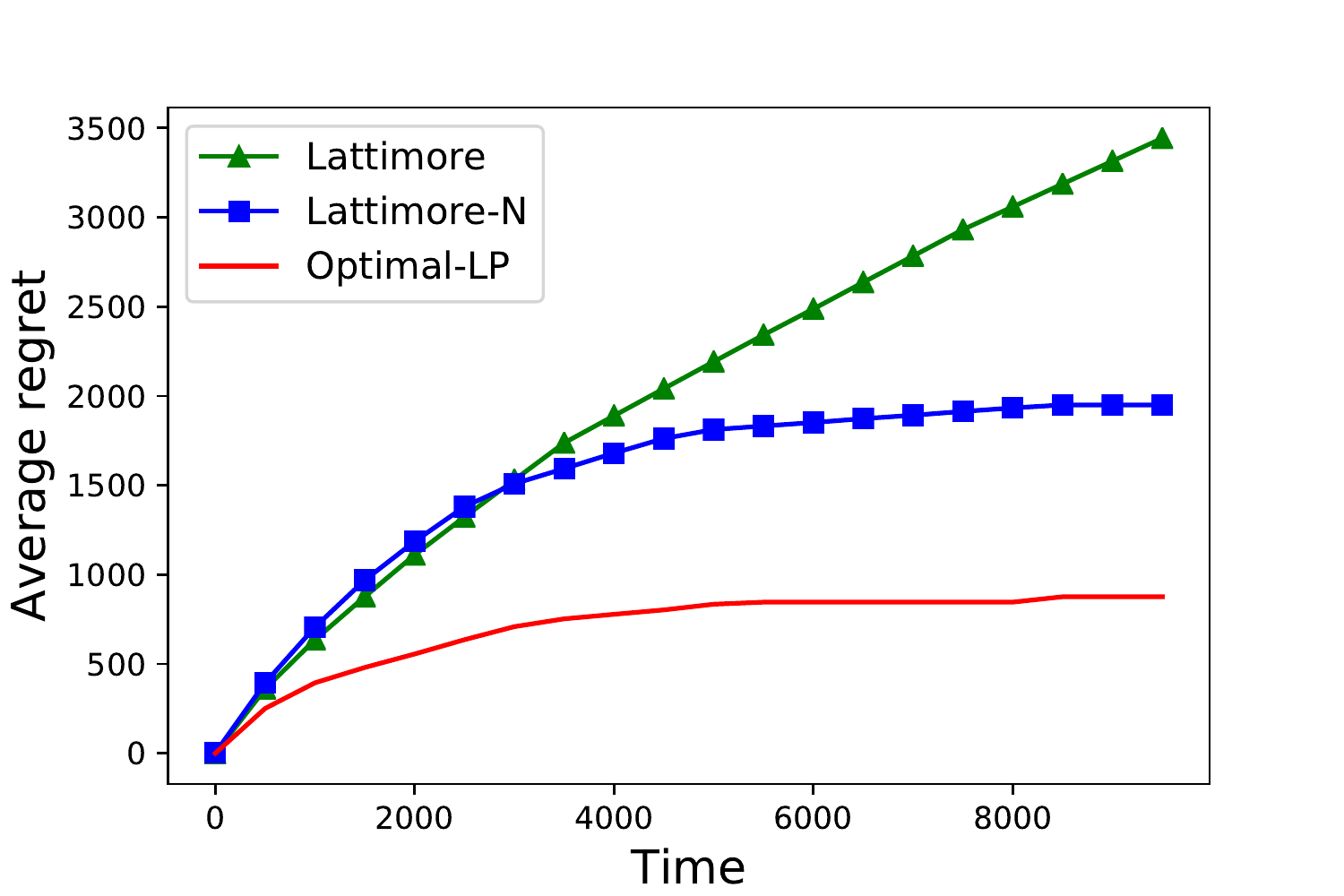}
     \caption{$d = 4, N =4,K=10, p =0.5 $}
     \label{fig:synthetic_2}
\end{subfigure}
\label{fig:synthetic_numerical}
\begin{subfigure}[b]{0.42\textwidth}
    \includegraphics[width=\textwidth]{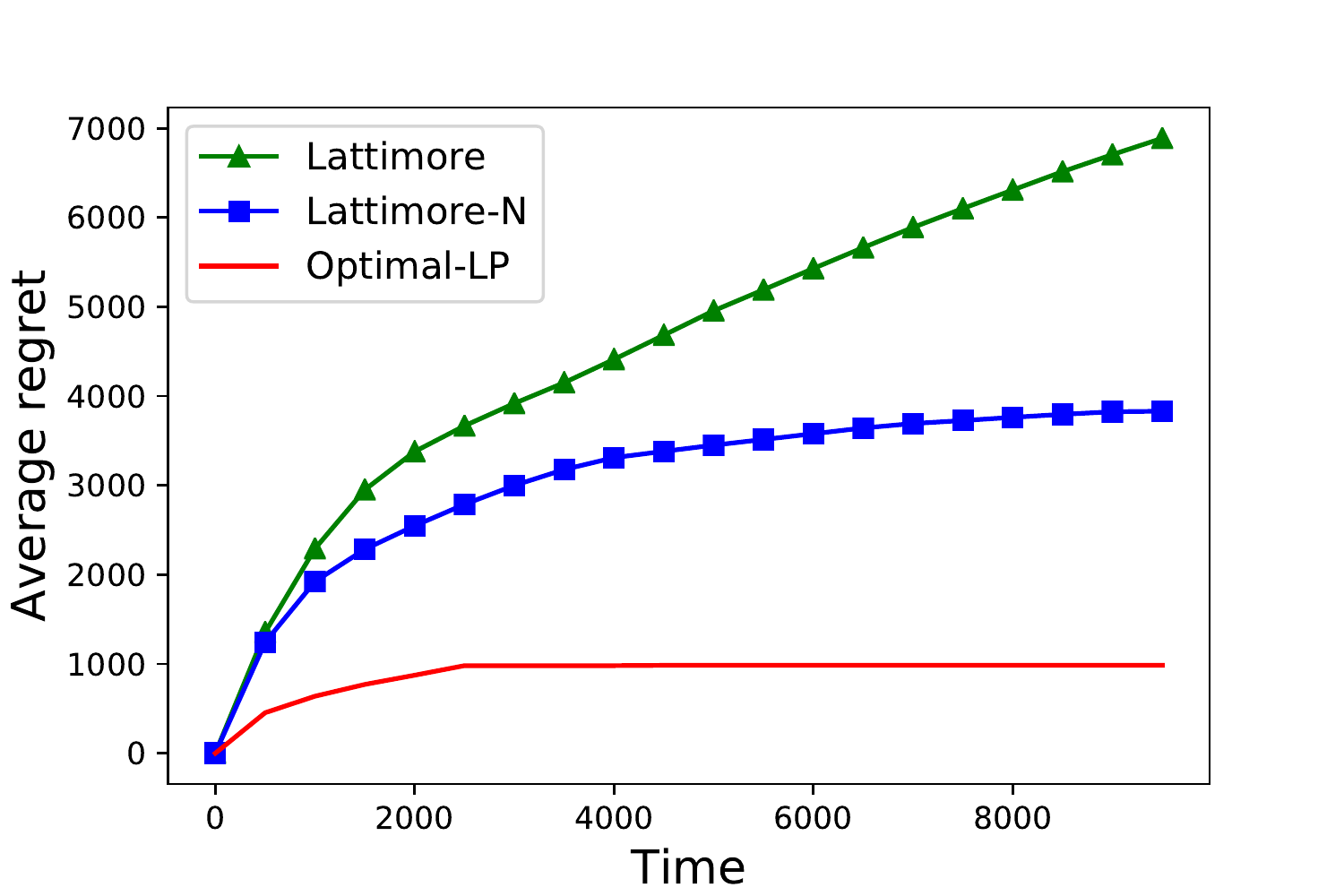}
     \caption{$d = 4, N =10,K=10, p =0.5 $}
     \label{fig:synthetic_3}
\end{subfigure}
\caption{Comparison of regret of Algorithm~\ref{algo:opt} with the algorithms  of~\cite{lattimore2016end}, which we denote as Lattimore and Lattimore-N. The plots are obtained after averaging the results of $100$ runs. $N$ and $K$ denote the number of users, and contexts respectively. The computational complexity of all the three algorithms is similar since they involve solving an optimization problem in which number of decision variables is equal to the number of arms.
}\label{fig:synthetic} 
\end{figure}
\emph{Synthetic Data Experiment}: 
The vector $\theta=\left\{\theta\ust_i\right\}_{i\in \cV}$ that contains the coefficients of the users, and the contexts of the arms, are generated randomly; $\theta_i$, and the context vectors are drawn from a uniform distribution with support in the set $[0,1]^{d}$. The edges in the graph $\cG$ are drawn randomly; any two nodes $i,j\in\cV$ are connected with a probability $p$. The noise $\eta_{i}(t),\eta_{(i,j)}(t)$ associated with the rewards and the side-observations~\eqref{def:reward},~\eqref{def:side_obs} are assumed to be Gaussian with standard deviation $0.1$. We compare the performance of Algorithm~\ref{algo:opt} with the algorithm of~\cite{lattimore2016end}, which is denoted Lattimore, and its adapation to the graphical setting which is denoted Lattimore-N. Lattimore-N is a naive adaptation of Lattimore algorithm, it uses side observations to enhance the estimation after warm-up phase, our experimental results show the potential gains from using the side observations alone. However, by leveraging the graph structure, our optimal algorithm shows significant regret reduction, and verifies our theoretical claims. We summarize the results of this evaluation in Figure~\ref{fig:synthetic}, where we plot the cumulative regret
of the algorithms as a function of rounds, averaged over $100$ runs.
\section{Discussion}
In this paper we introduce a framework to incorporate side-observations into the contextual bandits with linear payoff functions. We derive an instance-dependent lower bound on the regret of any learning algorithm, and also an optimal algorithm whose regret matches these lower bounds asymptotically as $T\to\infty$. We plan to extend the results to the case when the set of arms to choose from is non-stationary.
\bibliography{ref}
\onecolumn
\begin{center}
{\Large \textbf{Supplementary Materials} }
\end{center}
\appendix

\vspace{1cm}

\begin{center}
{\Large \textbf{Appendix} }
\end{center}
\appendix

\section{Auxiliary Results used in Proof of Theorem~\ref{th:main_res} }
The proof is composed of a sequence of lemmas that culminates in the proof of theorem.~The following result is Lemma~5 of~\cite{lattimore2016end}. 
\begin{lemma}\label{lemma:prob_count}
Let $\bP$ and $\bP^{\prime}$ be measures on the same measurable space $(\Omega,\cF)$. Then, for any event $A\in \cF$, we have, 
\begin{align*}
\bP(A)+\bP^{\prime}(A^{c}) \ge \frac{1}{2} \exp(-KL(\bP,\bP^{\prime}) ),
\end{align*} 
where $KL(\bP,\bP^{\prime})$ denotes the relative entropy between $\bP$ and $\bP^{\prime}$, which is defined as $+\infty$ if $\bP$ is not absolutely continuous with respect to $\bP^{\prime}$, and is equal to $\int_{\Omega} d\bP(\omega) \log \frac{d\bP}{d\bP^{\prime}}(\omega)$ otherwise. 
\end{lemma}

The following result shows that the relative
entropy between the probability measures induced by a learning algorithm on sequences of outcomes for two different multi-armed bandit problem instances can
be decomposed in terms of the expected number of times
each arm is chosen, and the relative entropies of the distributions
of the arms. We omit its proof since it is a minor modification of the proof of Lemma~6 of~\cite{lattimore2016end}. 

\begin{lemma}(Information Processing Lemma)\label{lemma:kl}
Consider a learning algorithm $\pi$ applied to two different problem instances, in which the users' coefficients are equal to $\{\theta\ust_i\}_{i\in\cV}$, and $\{\theta^{\prime}_i\}_{i\in\cV}$, while the graph and the set of arms are the same in both the instances and given by $\cG$ and $\cA$. Let $\bP,\bP^{\prime}$ denote the probability measures induced by $\pi$ on the sequence of rewards $\{r_i(s): i\in\cV\}_{s=1}^{t}$, side-observations $\{y_{(i,j)}(s):(i,j)\in\cE\}_{s=1}^{t}$ and actions $\{U_i(s): i\in\cV\}_{s=1}^{t} $. Furthermore, assume that $\theta\ust$ and $\theta^{\prime}$ differ only on the value at a single node $i$, i.e., $\theta\ust_j=\theta^{\prime}_j ,~\forall j\neq i$, and $\theta\ust_i \neq \theta^{\prime}_i$. Then we have the following,
\begin{align*}
KL(\bP,\bP^{\prime}) =  \frac{1}{2}  (\theta\ust_i-\theta^{\prime}_i)^{T} \bgi_{i} (T) ~(\theta\ust_i-\theta^{\prime}_i),
\end{align*}
where $\bar{G}_i(T)$ is as in~\eqref{def:git}, and the expectation is taken when $\theta\ust$ is the true parameter. 
\end{lemma}
\textbf{Constructing Modified Coefficient Vector $\theta^{\prime}$}\\
Recall that when the coefficient vector is equal to $\theta\ust$, arm $a\ust_i$ is the unique optimal arm for node $i$. We will now construct a coefficient vector ~$\theta^{\prime}$ so that the resulting optimal arm for node $i$ will be $b\ust$, where $b\ust\neq a\ust_i$. Since we do not modify the coefficients at other nodes, the optimal arms for other nodes $v\in\cV\setminus \{i\}$ remain unchanged. Let $H>0$ be a positive-definite matrix that will be specified soon. We let 
\begin{align}\label{def:theta_prime_1}
\theta^{\prime}_v= 
\begin{cases}
\theta\ust_v, ~~\mbox{ if } v\in\cV\setminus \{i\},\\
\\
\theta\ust_i + \frac{1}{\|b\ust- a\ust_i\|^{2}_{H}}  H(b\ust- a\ust_i)(\Delta_{b\ust}+\epsilon) \mbox{ if } v = i.
\end{cases}
\end{align} 
Note that under $\theta^{\prime}_i$, the mean reward of arm $b\ust$ is more than that of $a\ust_i$ since,
\begin{align}\label{eq:reward_bound}
(\theta^{\prime}_i)^{T}\left(b\ust - a\ust_i \right)&= \left(\theta\ust_i + \frac{1}{\|b\ust- a\ust_i\|^{2}_{H}}  H(b\ust- a\ust_i)(\Delta_{b\ust}+\epsilon)\right)^{T}\left( b\ust - a\ust_i\right)\notag\\
&=  -\Delta_{b\ust}+ \Delta_{b\ust}+\epsilon\notag\\
&=\epsilon.
\end{align}
Let $R^{\pi}_{(\theta\ust,\cG,\cA)}(T),R^{\pi}_{(\theta^{\prime},\cG,\cA)}(T)  $ denote the regret incurred by the learning algorithm $\pi$, when the coefficients are equal to $\theta\ust$ and $\theta^{\prime}$ respectively. We have the following lower-bound on $R^{\pi}_{(\theta\ust,\cG,\cA)}(T) +R^{\pi}_{(\theta^{\prime},\cG,\cA)}(T)$.
\begin{lemma}\label{lemma:sum_regrets}
Let $\theta^{\prime}$ be the coefficient vector constructed as in~\eqref{def:theta_prime_1}, and $\bP,\bP^{\prime}$ denote the probability measures induced by a learning algorithm $\pi$ on the sequence of rewards, side-observations and actions when users' coefficients in are equal to $\{\theta\ust_i\}_{i\in\cV}$, and $\{\theta^{\prime}_i\}_{i\in\cV}$ respectively. Furthermore, let $\epsilon < \Delta_{\min,i}$, where $\Delta_{\min,i}$ is as in~\eqref{def:del_mini}. We then have that,
\begin{align*}
R^{\pi}_{(\theta\ust,\cG,\cA)}(T) +R^{\pi}_{(\theta^{\prime},\cG,\cA)}(T) \ge \frac{\epsilon T}{2}\left[\bP\left(  N_{a\ust_i}(T)  \le T\slash 2   \right) + \bP^{\prime}\left(  N_{a\ust_i}(T)  > T\slash 2   \right)  \right],
\end{align*}
where $N_a(T)$ is the number of plays of arm $a$ until round $T$.  
\end{lemma}
\begin{proof}
Clearly,
\begin{align*}
R^{\pi}_{(\theta\ust,\cG,\cA)}(T) \ge \frac{T}{2}\Delta_{\min,i} \bP\left(  N_{a\ust_i}(T)  \le T\slash 2   \right).
\end{align*}
Similarly, it follows from~\eqref{eq:reward_bound} that
\begin{align*}
R^{\pi}_{(\theta^{\prime},\cG,\cA)}(T)  \ge \frac{T}{2} \epsilon~\bP^{\prime}\left(  N_{a\ust_i}(T)  \ge T\slash 2   \right).
\end{align*}
The proof then follows by adding the above two inequalities, and utilizing $\epsilon<\Delta_{\min,i}$. 
\end{proof}
\begin{lemma}
Let $\theta^{\prime}$ be the coefficient vector constructed as in~\eqref{def:theta_prime_1}, and $b\ust$ be the optimal arm at node $i$ under $\theta^{\prime}$. Define 
\begin{align}\label{def:delta_theta_j}
\delta a\ust:= b\ust- a\ust_i,
\end{align}
where $a\ust_i$ is the optimal arm for node $i$ when its coefficient is equal to $\theta\ust_{i}$. For $H>0$ define  
\begin{align}\label{def:rho_ij}
\rho_{i}(T;H):=   \|\delta a\ust\|^{2}_{H^{T}\bgi_{i} (T)H}\|\delta a\ust\|^{2}_{\bgin_{i}(T)}\left(\|\delta a\ust\|^{4}_{H}\right)^{-1}.
\end{align}
We then have that, 
\begin{align}\label{eq:2}
\frac{(\Delta_{b\ust}+\epsilon)^{2}}{2} \frac{\rho_{i}(T;H)}{\log T~\|\delta a\ust\|^{2}_{\bgin_{i}(T)}} 
\ge 1 + \frac{\log \epsilon-\log 2}{\log T}  - \frac{\log\left( R^{\pi}_{(\theta\ust,\cG,\cA)}(T) +R^{\pi}_{(\theta^{\prime},\cG,\cA)}(T)\right)}{\log T}.
\end{align}
\end{lemma}
\begin{proof}
It follows from Lemma~\ref{lemma:prob_count} and Lemma~\ref{lemma:sum_regrets} that
\begin{align}\label{ineq:3}
R^{\pi}_{(\theta\ust,\cG,\cA)}(T) +R^{\pi}_{(\theta^{\prime},\cG,\cA)}(T) \ge \frac{\epsilon T}{2}\exp(-KL(\bP,\bP^{\prime}) ).
\end{align}
Substituting the expression for $KL(\bP,\bP^{\prime})$ from Lemma~\ref{lemma:kl} into the above inequality, and taking logarithms, we obtain the following,
\begin{align*}
\frac{1}{2}  (\theta\ust_i-\theta^{\prime}_i)^{T} \bgi_{i} (T) ~(\theta\ust_i-\theta^{\prime}_i)\ge \log\left(\frac{\epsilon T}{2}\right) - \log\left( R^{\pi}_{(\theta\ust,\cG,\cA)}(T) +R^{\pi}_{(\theta^{\prime},\cG,\cA)}(T) \right).
\end{align*}
We then substitute the value of $\theta^{\prime}_i$ from~\eqref{def:theta_prime_1} in the above inequality and perform some algebraic manipulations in order to obtain~\eqref{eq:2}. 
\end{proof}
\begin{lemma}\label{coro:1}
Let $\delta a\ust$ be as in~\eqref{def:delta_theta_j}, and $\pi$ be a consistent learning algorithm. We then have that,
\begin{align*}
\liminf_{T\to\infty}\frac{\rho_{i}(T;H)}{\log T~\|\delta a\ust\|^{2}_{\bgin_{i}(T)}} \ge \frac{2}{\Delta^{2}_{b}},
\end{align*}
where $\rho_{i}(T;H)$ is as defined in~\eqref{def:rho_ij}.
\end{lemma} 
\begin{proof}
Since $\pi$ is a consistent learning algorithm, we have
\begin{align*}
\limsup_{T\to\infty}\frac{\log\left(R^{\pi}_{(\theta\ust,\cG,\cA)}(T) +R^{\pi}_{(\theta^{\prime},\cG,\cA)}(T) \right)}{\log T}\le 0.
\end{align*}
Substituting this into the inequality~\eqref{eq:2} yields 
\begin{align*}
\frac{(\Delta_{b\ust}+\epsilon)^{2}}{2} \liminf_{T\to\infty}\frac{\rho_{i}(T;H)}{\log T~\|\delta a\ust\|^{2}_{\bgin_{i} (T)}} \ge 1.
\end{align*}
The result then follows since the bound holds true for an arbitrary choice of $b\ust$, and for all $\epsilon>0$.
\end{proof}

Next, define 
\begin{align*}
c:= \limsup_{T\to\infty} \log T~\|\delta a\ust\|^{2}_{\bgin_i(T)},
\end{align*}
and let $d\in\bR$ be such that
\begin{align*}
d\le \liminf_{T\to\infty}\frac{\rho_{i}(T;H)}{\log T~\|\delta a\ust\|^{2}_{\bgin_{i}(T)}}. 
\end{align*}
We then have that 
\begin{align}\label{ineq:9}
c \le \frac{\liminf_{T\to\infty} \rho_{i}(T;H)}{d},
\end{align}
where $H>0$. It follows from Lemma~\ref{coro:1} that $d$ can be taken to be $2\slash \Delta^{2}_{b\ust}$. We now obtain an upper-bound on $\liminf_{T\to\infty} \rho_{i}(T;H)$ which will give us an upper-bound on $c$.
\begin{lemma}\label{lemma:regret_lb_sum}
Define,
\begin{align}
\tilde{H}_{i}(T) := \frac{\bgin_{i}(T)}{\|\bgin_{i}(T)\|},
\end{align} 
and let $\tilde{H}_{i}(\infty)$ be a limit point of $\tilde{H}_{i}(T)$. 
We then have that 
\begin{align}\label{ineq:10}
\liminf_{T\to\infty}\rho_{i}(T;\tilde{H}_{i}(\infty)) \le 1.
\end{align}
\end{lemma}
\proof

We have
\begin{align*}
\rho_{i}(T;H) &=   \frac{\|\delta a\ust\|^{2}_{H^{T}\bar{G}_i(T)H}\|\delta a\ust\|^{2}_{\bar{G}^{-1}_i(T)}}{\|\delta a\ust\|^{4}_{H}}\\
&=\| \delta a\ust \|^{2}_{\tilde{H}_i(T)}~~\| \delta a\ust \|^{2}_{H\tilde{H}_i(T)H} ~~\|\delta a\ust\|^{-4}_{H}.
\end{align*}
The last expression computes to $1$ with $H$ set equal to $\tilde{H}_{i}(T)$. It then follows that
\begin{align*}
\liminf_{T\to\infty}\rho_{i}(T;\tilde{H}_{i}(\infty)) \le 1.
\end{align*}
\begin{lemma}\label{lemma:bound}
Under any consistent learning algorithm $\pi$, we have that
\begin{align*}
 \limsup_{T\to\infty} \log T~\|b- a\ust\|^{2}_{\bgin_{i}(T)} \le \frac{\Delta^{2}_{b}}{2},~~\forall b \in \cA^{(s)}_{i}.
\end{align*}
\end{lemma}
\begin{proof} Follows by substituting~\eqref{ineq:10} into the inequality~\eqref{ineq:9}, and choosing $d$ to be equal to $2\slash \Delta^{2}_{b}$.  
\end{proof}

We are now in a position to prove Theorem~\ref{th:main_res}. 

\section{Proof of Theorem~\ref{th:1}}
We begin by deriving bounds on the error associated with the estimates $\hat{\theta}_i(t)$ that are obtained as in~\eqref{def:empirical}, i.e., 
$\hat{\theta}_i(t) =G^{-1}_i(t) \left[\sum_{s=1}^{t}\sum_{j\in\cN_i} y_{(j,i)}(s)U_{j}(s)\right]$. Substituting the expressions for rewards $r_i(s)$ and side-observations $y_{(j,i)}(s)$ from~\eqref{def:reward} and~\eqref{def:side_obs}, we obtain the following,
\begin{align}\label{def:error}
e_i(t) :&= \hat{\theta}_i(t) -\theta\ust_i \notag\\
&= G^{-1}_i(t) \sum_{s=1}^{t}\sum_{j\in\cN_i}\eta_{(j,i)}(s) U_j(s). 
\end{align}
For $x\in\bR^{d}$, consider:
\begin{align}\label{eq:1}
x^{T}e_i(t) = \sum\limits_{s=1}^{t}\sum_{j\in\cN_i}\eta_{(j,i)}(s)x^{T}G^{-1}_i(t)U_j(s).
\end{align}
Define the following ``error event,'' 
\begin{align}\label{def:eix}
\cE_i(x,\alpha,t) := \left\{\omega:  | x^{T}e_i(t)| > \alpha  \right\}, \mbox{ where }\alpha>0, i\in\cV,~t\in[1,T].
\end{align}

\begin{lemma}\label{lemma:1}
Let the decisions $\{U_i(t): i\in\cV\}_{t\in[1,T]}$ be deterministic. 
We then have that,
\begin{align}\label{ineq:15}
\bP\left( \cE_i(x,\alpha,t) \right) \le 2\exp\left(-\frac{\alpha^2}{2 \|x\|^{2}_{ G^{-1}_i(t)}}  \right).
\end{align}
\end{lemma}
\begin{proof}
For $\lambda>0$, it follows from Chebyshev's inequality that,
\begin{align}\label{ineq:cheby}
\bP\left( x^{T}e_i(t) >\alpha  \right)  \le \exp (-\lambda \alpha) \bE \exp (\lambda x^{T}e_i(t)). 
\end{align}
Substituting the expression for $x^{T}e_i(t)$ from~\eqref{eq:1}, we obtain,
\begin{align}
\bE \exp (\lambda x^{T}e_i(t)) = \exp\left(\frac{\lambda^2}{2} \|x\|^{2}_{G^{-1}_i(t)} \right).
\end{align}
Substituting the above into the inequality~\eqref{ineq:cheby}, we obtain
\begin{align*}
\bP\left( x^{T}e_i(t) >\alpha  \right) \le \exp (-\lambda \alpha) \exp\left(\frac{\lambda^2}{2} \|x\|^{2}_{G^{-1}_i(t)} \right).
\end{align*}
For $\lambda = \alpha \slash \|x\|^{2}_{G^{-1}_i(t)}$, the above inequality reduces to
\begin{align}\label{ineq:ldp}
\bP\left( x^{T}e_i(t) >\alpha  \right) \le \exp\left(-\frac{\alpha^2}{2 \|x\|^{2}_{G^{-1}_i(t)}}  \right).
\end{align}
A similar bound can be derived for the probability of the event $\left\{x^{T}e_i(t) <-\alpha \right\}$. Combining these bounds completes the proof. \end{proof}

Note that the exploration phase is composed of ``episodes,'' such that each arm in the spanner $\cS$ is played exactly once during an episode.
Thus, an episode lasts for $d$ rounds. After $t$ episodes, the matrices $G_i(td)$ are given as follows,
\begin{align*}
G_i(td) = t\left( \sum_{j\in\cN_i} \sum_{a\in\cS\cap \cA_j} aa^{T} \right) = t|\cN_i|\left( \sum_{u\in\cC} u~u^{T}\right),
\end{align*}
so that
\begin{align}\label{ineq:11}
G^{-1}_i(td) =\frac{1}{t|\cN_i|}\left( \sum_{u\in\cC} u~u^{T}\right)^{-1}.
\end{align}

\begin{lemma}\label{prop:2}
If the decisions $\left\{U_i(t): i\in\cV\right\}_{t\in[1,T]}$ are such that only the arms in $\cS$ are played in a round-robin manner, then, 
\begin{align}\label{ineq:18}
\|a\|^{2}_{G^{-1}_{i_a}(kd)}\le \frac{d}{k|\cN_i|} ,~ \forall a\in\cA,~ k= 1,2,\ldots, \lfloor T\slash d \rfloor,
\end{align}
where for $x\in \bR$, $\lfloor x\rfloor$ denotes the greatest integer less than or equal to $x$. 
\end{lemma}
\begin{proof} 
Within this proof, we let $i$ denote the node $i_a$ at which arm $a$ is played. Since $\cC$ is a barycentric spanner for the set of context vectors $\cU$ (see Definition~\ref{def:bs1}), we have $u_a = \sum_{u\in \cC}\alpha_{u} u$, where $\alpha_{u}\in [-1,1],~\forall u\in\cC$. Thus, 
\begin{align*}
\|u_a\|^{2}_{G^{-1}_{i}(kd)} &= \sum_{u\in \cC} \alpha^{2}_{u}~ u^{T}  G^{-1}_{i}(kd)~u\\
&\le \frac{1}{k|\cN_i|}\sum_{u\in \cC} \alpha^{2}_{u}~ u^{T} \left(u~u^{T}\right)^{\dagger} u\\
&\le \frac{1}{k|\cN_i|}\sum_{u\in \cC} u^{T} \left(u~u^{T}\right)^{\dagger} u,\\
&\le \frac{d}{k|\cN_i|}, 
\end{align*}
where for a matrix $A$, we let $A^{\dagger}$ denote its pseudoinverse, the first inequality follows from~\eqref{ineq:11}, and the second inequality follows since $|\alpha_u|\le 1$. 
\end{proof}
Recall that the size of confidence intervals, $\alpha(t)$, is as follows,
\begin{align}\label{def:alpha}
\alpha(t) = \sqrt{\frac{2\log\left(T \sum_{i\in\cV} |\cA_i|\slash \delta\right)}{t }~d}.
\end{align}
\begin{lemma}\label{lemma:4}
Define
\begin{align}\label{def:good_set}
\cE := \cup_{k\in\lfloor1,T\slash d\rfloor,i\in\cV,a_i\in\cA_{i}}\cE_i(a_i,\alpha(kd),kd),
\end{align}
where $\alpha(t)$ is as in~\eqref{def:alpha}. We have the following upper-bound while playing the arms in the set $\cS$ in a round-robin manner, 
\begin{align}\label{ineq:unif_bound}
\bP\left(\cE \right) \le \delta.
\end{align}
\end{lemma}
\begin{proof} 
Substituting the bound~\eqref{ineq:18} for $\|a\|^{2}_{G^{-1}_{i_a}(td)}$ into the inequality~\eqref{ineq:15}, we obtain,
\begin{align*}
\bP\left( \cE_i(a_i,\alpha(kd),kd) \right) &\le \exp\left(-\frac{\alpha^2(kd) |\cN_i| kd }{2 d}  \right)\\
&\le \frac{\delta}{T \sum_{i\in\cV} |\cA_i| }.
\end{align*}
The proof then follows by using the union bound, i.e., $\bP(\cE) \le \sum_{(k,i,a_i)} \bP\left(\cE_i(a_i,\alpha(kd),kd) \right)$. 
\end{proof}

\begin{lemma}\label{lemma:topmost}
Consider the set $\cE$ as defined in~\eqref{def:good_set}, and the stopping-time $\tau_i$ in~\eqref  {def:stop_1}. In the set $\cE^{c}$, we have that $\hat{a}\ust_i(\tau_i)= a\ust_i,~\forall i\in \cV$, i.e., the estimates of the optimal arms are equal to the true optimal arms. Thus, with a probability greater than $1-\delta$, we have $\hat{a}\ust_i(\tau_i)= a\ust_i,~\forall i\in \cV$.
\end{lemma}
\begin{proof} 
Note that $\hat{a}\ust_i(\tau_i)$ is the arm that corresponds to $\cB^{(o)}_{i,1}(\tau_i)$. In $\cE^{c}$, we have $\mu_{a\ust_i}\in \cB_{a\ust_i}(\tau_i)$, and also $\mu_{a}\in \cB_{a}(\tau_i)$ for any sub-optimal arm $a\in\cA^{(s)}_i$. This means that in $\cE^{c}$, the ball $\cB^{(o)}_{i,1}(\tau_i)$ is equal to the ball $\cB_{a\ust_i}(\tau_i)$, since if this was not the case, then we would have a contradiction that $\mu_{a} > \mu_{a\ust_i}$ for some sub-optimal arm $a$. Hence we conclude that $\hat{a}\ust_i(\tau_i)= a\ust_i$ on $\cE^{c}$. The proof is completed by noting that from~\eqref{ineq:unif_bound}, we have that $\bP(\cE^{c})\ge 1-\delta$.
\end{proof}
Next, we derive an upper-bound on the stopping time $\tau$ that marks the end of the exploration phase. This yields an upper-bound on the ``exploration regret.''
\begin{lemma}\label{lemma:6}
Consider the stopping time $\tau$ as defined in~\eqref{def:stop_1}-\eqref{def:stop_2}. The following holds true on the set $\cE^{c}$,
\begin{align*}
\tau \le \frac{2\log\left(T \sum_{i\in\cV} |\cA_i|\slash \delta\right)d}{\left(\Delta_{\min}\slash 2\right)^{2}}.
\end{align*}
\end{lemma}
\begin{proof}
We begin by deriving an upper-bound on $\tau_i$. Note that on $\cE^{c}$, the mean rewards of arms lie within their corresponding confidence balls. Hence in order for the ball $\cB_{a\ust_i}$ and the ball $\cB_{a}$, that corresponds to a sub-optimal arm $a\in\cA^{(s)}_{i}$, to intersect during round $t$, we must necessarily have that,
\begin{align*}
\mu_{a\ust_i} -\alpha(t) &\le \mu_{a} + \alpha(t), a\in\cA^{(s)}_i,
\end{align*}
which gives
\begin{align*}
\alpha(t) \ge  \frac{\Delta_{a}}{2},~a\in\cA^{(s)}_i.
\end{align*}
Substituting the expression for $\alpha(t)$ from~\eqref{def:alpha} into the above inequality, we obtain,
\begin{align*}
\sqrt{\frac{2\log\left(T \sum_{i\in\cV} |\cA_i|\slash \delta\right)}{t }d}~ \ge \frac{\Delta_{a}}{2},~a\in\cA^{(s)}_i,
\end{align*}
or 
\begin{align*}
 t \le \frac{2\log\left(T \sum_{i\in\cV} |\cA_i|\slash \delta\right)d}{\left(\Delta_{a}\slash 2\right)^{2}}, a\in\cA^{(s)}_i.
\end{align*}
We deduce from the above inequality that in $\cE^{c}$, the ball $\cB_{a\ust_i}(t)$ cannot intersect with $\cB_{a},a\in\cA^{(s)}_{i}$ during rounds $t > \frac{2\log\left(T \sum_{i\in\cV} |\cA_i|\slash \delta\right)d}{\left(\Delta_{\min,i}\slash 2\right)^{2}}, a\in\cA^{(s)}_i$. Thus, $\tau_i \le \frac{2\log\left(T \sum_{i\in\cV} |\cA_i|\slash \delta\right)d}{\left(\Delta_{\min,i}\slash 2\right)^{2}}$. The proof then follows by noting that $\tau = \max_i \tau_i$. \end{proof}

\begin{proof}[Proof of Theorem~\ref{th:1}]
It follows from Lemma~\ref{lemma:topmost} that on the set $\cE^{c}$, the regret of Algorithm~\ref{algo:stop} is $0$ during the rounds $t$ greater than time $\tau$. Thus, we have,
\begin{align*}
\mathbbm{1}\left( \cE^{c}\right)\sum_{t=1}^{T}\sum_{i\in\cV}\Delta_{U_i(t)} \le \tau \left(\sum_{i\in\cV} \Delta_{\max,i} \right).
\end{align*}
Upon substituting the upper-bound on $\tau$ that was derived in Lemma~\ref{lemma:6}, we obtain the following,
\begin{align}\label{ineq:20}
\bE\left( \mathbbm{1}\left( \cE^{c}\right)\sum_{t=1}^{T}\sum_{i\in\cV}\Delta_{U_i(t)} \right)\le \frac{2\log\left(T \sum_{i\in\cV} |\cA_i|\slash \delta\right)d}{\left(\Delta_{\min}\slash 2\right)^{2}} \left(\sum_{i\in\cV} \Delta_{\max,i} \right).
\end{align}
Moreover, since the cumulative regret on any sample path is trivially upper-bounded by $T\left(\sum_{i\in\cV} \Delta_{\max,i} \right)$, we have that,
\begin{align}\label{ineq:21}
\bE\left( \mathbbm{1}\left( \cE\right)\sum_{t=1}^{T}\sum_{i\in\cV}\Delta_{U_i(t)} \right)&\le T\left(\sum_{i\in\cV} \Delta_{\max,i} \right)\bP(\cE)\notag\\
&\le \delta T\left(\sum_{i\in\cV} \Delta_{\max,i} \right),
\end{align}
where, the last inequality follows from Lemma~\ref{lemma:4}.~The proof then follows by combining the inequalities~\eqref{ineq:20} and~\eqref{ineq:21}.
\end{proof}

\section{Auxiliary Results used in Proof of Theorem~\ref{th:regret_opt}}
\subsection{Preliminary Results}
We begin by deriving some results that will be useful while analyzing the regret of Algorithm~\ref{algo:opt}. Recall that $\hat{\Delta}=\left\{ \hat{\Delta}_a: a\in\cA \right\}$ denotes the estimates~\eqref{def:sog} of sub-optimality gaps of arms, obtained at the end of the warm-up phase.
\begin{lemma}\label{lemma:7}
Consider the optimization problem OPT$(\hat{\Delta})$~\eqref{def:lph_o}-\eqref{def:lph_c2}, solving which requires the estimates $\hat{\Delta}$ as an input. We then have that,
\begin{align*}
\sum_{a\in\cA^{(s)}} \beta\ust_{a}(\hat{\Delta}) \le 2d^{3} f(T) \frac{\hat{\Delta}_{\max}}{\hat{\Delta}^{3}_{\min}},
\end{align*}
where $\beta\ust(\hat{\Delta}) =\left\{\beta\ust_a(\hat{\Delta})\right\}_{a\in\cA}$ is a solution of~\eqref{def:lph_o}-\eqref{def:lph_c2}.
\end{lemma}
\begin{proof}
We omit the proof since it follows closely the proof of Lemma 12 of~\cite{lattimore2016end}.
\end{proof}
 
\begin{lemma}\label{lemma:10}
Define $\delta_T$ as follows,
\begin{align}\label{def:delta_T}
1+\delta_T := \max_{a\in\cA: \hat{\Delta}_a >0 }~~ \frac{\Delta^{2}_a}{\hat{\Delta}^{2}_a}.
\end{align}
We then have that,
\begin{align}\label{ineq:13}
\sum_{a\in\cA} \beta\ust_{a}(\hat{\Delta}) \hat{\Delta}_{a} \le
(1+\delta_T)\sum_{a\in\cA} \beta\ust_{a}(\Delta) \hat{\Delta}_{a}.
\end{align}
\end{lemma}
\begin{proof} For an arm $a$, we have
\begin{align}
\|a\|^{2}_{H^{-1}_{i_a}((1+\delta_T)\beta\ust(\Delta) )  } &= \frac{\|a\|^{2}_{H^{-1}_{i_a}(\beta\ust(\Delta) )  } }{(1+\delta_T)}\notag\\
&\le \frac{\Delta^{2}_a}{ (1+\delta_T)f(T)}\notag\\
&\le \frac{\hat{\Delta}^{2}_a}{f(T)},\label{ineq:19}
\end{align}
where $H_i(\cdot)$ is as defined in~\eqref{def:lph_c2}, the first inequality follows since $\beta\ust(\Delta)$ is feasible for $OPT(\Delta)$, and the last inequality follows from the definition of $\delta_T$. It follows from inequality~\eqref{ineq:19} that the vector $\left\{(1+\delta_T)\beta\ust_a(\Delta):a\in\cA \right\}$ is feasible for $OPT(\hat{\Delta})$.~Hence, the optimal value of $OPT(\hat{\Delta})$ is upper-bounded by $(1+\delta_T)\sum_{a\in\cA} \beta\ust_{a}(\Delta) \hat{\Delta}_{a}$. This completes the proof. 
\end{proof}

\begin{lemma}\label{lemma:8}
If $2\epsilon_{T}(d\log^{1\slash 2} T) \le \Delta_{\min} \slash 2$, then we have the following upper-bound on the quantity $\delta_T$ that was defined in~\eqref{def:delta_T},
\begin{align*}
\delta_T \le \frac{16 \epsilon_{T}(d\log^{1\slash 2} T) }{\Delta_{\min}}.
\end{align*}
\end{lemma}
\begin{proof} 
Within this proof we use $\epsilon_T$ to denote $\epsilon_{T}(d\log^{1\slash 2} T)$. We have
\begin{align*}
1+\delta_T = \max_{a\in\cA: \hat{\Delta}_a >0 } \frac{\Delta^{2}_a}{\hat{\Delta}^{2}_a} \le \max_{a\in\cA: \hat{\Delta}_a >0 } \frac{\Delta^{2}_a}{\left(\Delta_a - 2\epsilon_{T}\right)^{2} } \le \max_{a\in\cA: \hat{\Delta}_a >0 }\left( 1 + \frac{4\left( \Delta_a - \epsilon_T\right)\epsilon_T}{\left(\Delta_a - 2\epsilon_{T}\right)^{2} }   \right) \le 1 + \frac{16\epsilon_T}{\Delta_{\min}}.
\end{align*}
\end{proof}

Consider the following two events:
\begin{align}
F:&= \bigcup\limits_{a\in\cA,t\in[1,T]}\left\{ \omega: | \mu_{a} - \hat{\mu}_{a}(t)  |  \ge  \|a\|_{G^{-1}_{i_a}(t) } g^{1\slash 2}(T)   \right\},\label{def:f}\\
F^{\prime} :&= \bigcup\limits_{a\in\cA,t\in[1,T]}\left\{ \omega:  | \mu_{a} - \hat{\mu}_{a}(t)  |  \ge  \|a\|_{G^{-1}_{i_a}(t) } f^{1\slash 2}(T)    \right\},\label{def:fp}
\end{align}
where, the functions $f(\cdot)$ and $g(\cdot)$ are as in~\eqref{def:ft}-\eqref{def:gt}. The following result is essentially Theorem 8 of~\cite{lattimore2016end}.
\begin{lemma}\label{lemma:5}
Consider the operation of Algorithm~\ref{algo:opt} on an instance of the contextual bandit problem with side-observations. Let the events $F$ and $F^{\prime}$ be as defined in~\eqref{def:f}-\eqref{def:fp}. We then have that,
\begin{align*}
\bP(F) \le \frac{1}{\log\left(T\right)}, ~~\bP(F^{\prime})\le \frac{1}{T}. 
\end{align*} 
\end{lemma}

\subsection{Regret Analysis of Algorithm~\ref{algo:opt}}
We analyze the regret on the following three sets separately: 
(i) $F^{c}$, (ii) $F\cap \left( F^{\prime} \right)^{c}$, and (iii) $F^{\prime}$. 

\emph{Regret Analysis on $F^{c}$}
\begin{lemma}\label{lemma:2}
Algorithm~\ref{algo:opt} never enters recovery phase on $F^{c}$. 
\end{lemma}
\begin{proof} It follows from~\eqref{def:f} that on $F^{c}$ we have the following,
\begin{align*}
| \mu_{a} - \hat{\mu}_{a}(t)  | \le \|a\|_{G^{-1}_{i_a}(t) } g^{1\slash 2}(T)  \le \epsilon_{T}(t),~\forall a\in\cA.
\end{align*}
Thus, for times $s,t \ge d\log^{1\slash 2}T$, we have
\begin{align*}
| \hat{\mu}_{a}(s) - \hat{\mu}_{a}(t)  | \le 2\epsilon_{T}(\min \{s,t\} )\le 2\epsilon_{T}(d\log^{1\slash 2} T). 
\end{align*}
Since recovery phase occurs only when $|\hat{\mu}_a(t) -\hat{\mu}_a| > 2\epsilon_T(d\log^{1\slash 2} T)$, where $\hat{\mu}_a$ is the estimate of $\mu_a$ at time $d\log^{1\slash 2}T$, this shows that the algorithm does not enter the recovery phase on $F^{c}$. 
\end{proof}
\begin{lemma}\label{lemma:reg_fc}
The cumulative regret of Algorithm~\ref{algo:opt} during the success phase, in the set $F^{c}$, can be bounded as follows,
\begin{align}\label{ineq:7}
\limsup_{T\to\infty}\bE\left[ \frac{\mathbbm{1}\left(F^{c}\right)  \sum_{t\in \cT_{succ} } \sum_{i\in\cV} \Delta_{U_i(t)}}{\log T} \right] \le c(\theta\ust,\cG,\cA),
\end{align}
where $c(\theta\ust,\cG,\cA)$ is the optimal value of the problem~\eqref{lp_1}-\eqref{lp_3}.
\end{lemma}
\begin{proof}
Within this proof, we use $\hat{\Delta}$ and $\epsilon_{T}$ in lieu of $\hat{\Delta}(d\log^{1\slash 2} T )$ and $\epsilon_{T}(d\log^{1\slash 2}(T) )$. Recall that $\beta\ust(\Delta)$ is the number of plays of arms calculated by solving the optimization problem~\eqref{def:lph_o}-\eqref{def:lph_c2}.  $\beta\ust(\Delta)$ satisfies the following, 
\begin{align}\label{ineq:12}
\limsup_{T\to\infty} \frac{\sum_{a\in\cA^{(s)}} \beta\ust_{a}( \Delta  )\Delta_a}{\log T} = c(\theta\ust,\cG,\cA).
\end{align}
The regret occured during the success phase satisfies, 
\begin{align}\label{ineq:14}
\mathbbm{1}\left(F^{c}\right) \sum_{t\in\cT_{succ}} \sum_{i\in\cV} \Delta_{U_i(t)} &\le  \sum_{a\in\cA^{(s)}} \beta\ust_{a}(\hat{\Delta}) \Delta_{a}\notag\\
&= \sum_{a\in\cA^{(s)}} \beta\ust_{a}(\hat{\Delta}) \hat{\Delta}_{a} +  \sum_{a\in\cA^{(s)}} \beta\ust_{a}(\hat{\Delta}) \left[\Delta_{a} -\hat{\Delta}_{a} \right]\notag\\
&\le (1+\delta_T)\sum_{a\in\cA^{(s)}} \beta\ust_{a}(\Delta) \hat{\Delta}_{a} + 2\epsilon_{T} \sum_{a\in\cA^{(s)}} \beta\ust_{a}(\hat{\Delta})\notag \\
&\le (1+\delta_T)\sum_{a\in\cA^{(s)}} \beta\ust_{a}(\Delta) \Delta_{a} \notag\\
&+ 2\epsilon_{T}(d\log^{1\slash 2}(T)) \left[(1+\delta_T)\sum_{a\in\cA^{(s)}} \beta\ust_{a}(\Delta) + \sum_{a\in\cA^{(s)}} \beta\ust_{a}(\hat{\Delta}) \right],
\end{align}
where the first inequality follows since under Algorithm~\ref{algo:opt}, the number of plays of an arm $a$ is atmost equal to $\beta\ust_{a}(\hat{\Delta})$, the second inequality follows from~\eqref{ineq:13} and the fact that on $F^{c}$ we have $| \mu_{a} - \hat{\mu}_{a}(t)  | \le \epsilon_{T}$. We now use the results of Lemma~\ref{lemma:7} and Lemma~\ref{lemma:8} in the inequality~\eqref{ineq:14}, and also choose the operating horizon $T$ to be sufficiently large so as to satisfy $2\epsilon_{T} \le \Delta_{\min} \slash 2$, and obtain the following,
\begin{align}\label{ineq:22}
\mathbbm{1}\left(F^{c}\right) \sum_{t\in\cT_{succ}} \sum_{i\in\cV} \Delta_{U_i(t)} \le \left(1+\frac{16 \epsilon_{T} }{\Delta_{\min}} \right)\sum_{a\in\cA^{(s)}} \beta\ust_{a}(\Delta) \Delta_{a} + 2\epsilon_{T} \left[2+ 2d^{3} f(T) \frac{\Delta_{\max}}{\Delta^{3}_{\min}} \right].
\end{align}
We have
\begin{align}\label{ineq:bound_eps}
\epsilon_{T} = O\left(  \frac{\log^{1\slash 2} \left(\log T\right)}{\log^{1\slash 4} T}    \right).
\end{align}
We now divide both sides of the inequality~\eqref{ineq:22} by $\log T$, and substitute~\eqref{ineq:bound_eps} in~\eqref{ineq:22} in order to obtain the following, 
\begin{align}
\limsup_{T\to\infty}\frac{1}{\log T}\mathbbm{1}\left(F^{c}\right) \sum_{t\in\cT_{succ}} \sum_{i\in\cV} \Delta_{U_i(t)} \le \limsup_{T\to\infty}\frac{1}{\log T} \sum_{a\in\cA^{(s)}} \beta\ust_{a}(\Delta) \Delta_{a}\le c(\theta\ust,\cG,\cA) ,
\end{align}
where the last inequality follows from~\eqref{ineq:12}. The proof then follows from Fatou's lemma~\cite{folland2013real}.  
\end{proof}


\emph{Regret Analysis on $F\cap \left( F^{\prime} \right)^{c}$}: 
\begin{lemma}\label{lemma:3}
We have,
\begin{align}\label{ineq:8}
\limsup_{T\to\infty}\bE\left[\frac{\mathbbm{1}\left(F\cap \left(F^{\prime}\right)^{c}\right)\sum_{t\in \cT_{succ} } \sum_{i\in\cV} \Delta_{U_i(t)}}{\log T} \right] =0.
\end{align}
\end{lemma}
\begin{proof} 
We omit the proof since it follows closely the proof of Lemma 13 of~\cite{lattimore2016end}.
\end{proof}
\subsection{Proof of Corollary~\ref{coro:2}}
Consider the following optimization problem 
\begin{align}
OPT_1: \min_{\left\{w_a\right\}_{a\in\cA}} &\Delta_{\max} \sum_{a\in\cA}  w_a \label{lptilde_1}\\
\mbox{ s.t. } n_i(w) &\ge \frac{\sqrt{2}}{\Delta_{\min}},  ~~\forall i \in\cV,\label{lptilde_2}\\
\mbox{ where }n_i(w) :&= \sum_{j\in\cN_i} \sum_{\left\{a: i_a = j\right\}} w_a,~i\in\cV.\label{lptilde_3},\\
w_a = 0 \mbox{ if } a\notin \cS, &\mbox{ and }w_a = w_b \forall a,b\in\cS\cap\cA_{i_a}. 
\end{align}
It is easily verified that any vector feasible for $OPT_1$ is also feasible for $OPT$. Moreover, its objective function is also greater than the objective of $OPT$. Thus, its optimal value, denoted $c(\cA,\theta\ust,\cG)_1$ is greater than $c(\cA,\theta\ust,\cG)$. Consider now a scaled version of $OPT_1$,
\begin{align}
OPT_{1,s}: \min_{\left\{w_a\right\}_{a\in\cA}} & \sum_{a\in\cA}  w_a \label{lps_1}\\
\mbox{ s.t. } n_i(w) &\ge 1,  ~~\forall i \in\cV,\label{lps_2}\\
\mbox{ where }n_i(w) :&= \sum_{j\in\cN_i} \sum_{\left\{a: i_a = j\right\}} w_a,~i\in\cV\label{lps_3},\\
w_a = 0 \mbox{ if } a\notin \cS, &\mbox{ and }w_a = w_b,~\forall a,b\in\cS\cap\cA_{i_a}. \label{lps_4}
\end{align}
It is evident that if $x$ is feasible for $OPT_1$, then $x \Delta_{\min}\slash \sqrt{2}$ is feasible for $OPT_{1,s}$, and if $y$ is feasible for $OPT_{1,s}$, then $y\sqrt{2}\slash \Delta_{\min}$ is feasible for $OPT_{1}$. Thus, if $c(\cA,\theta\ust,\cG)_{1,s}$ denotes optimal value of $OPT_{1,s}$, then we have $c(\cA,\theta\ust,\cG)_{1} = c(\cA,\theta\ust,\cG)_{1,s}\left(\sqrt{2}\slash \Delta_{\min}\right)\Delta_{\max}$. The proof is completed by noting that the optimal value of $OPT_{1,s}$ is a lower bound on $|\chi(\cG)|$.

\end{document}